\documentclass[final,5p,times,twocolumn]{elsarticle}

\usepackage{amsmath,amsfonts,amssymb}
\usepackage{amsthm}
\usepackage{booktabs}
\usepackage{array}
\usepackage{algorithm}
\usepackage{algorithmic}
\usepackage{graphicx}
\usepackage{url}
\usepackage[hidelinks]{hyperref}
\usepackage{caption}

\newtheorem{assumption}{Assumption}
\newtheorem{proposition}{Proposition}

\journal{Engineering Applications of Artificial Intelligence}

\begin{document}
	
	\begin{frontmatter}
		
		\title{Verification-Gated Agentic Mission-State Governance for Intelligent Industrial Multi-Robot Systems}
		
		\author[inst1]{Guoqin Tang}
		\author[inst1]{Qingxuan Jia}
		\author[inst1]{Yichen Tan}
		\author[inst1]{Zeyuan Huang}
		\author[inst1]{Ning Ji}
		\author[inst1]{Gang Chen\corref{cor1}}
		\ead{buptcg@163.com}
		
		\cortext[cor1]{Corresponding author.}
		
		\affiliation[inst1]{
			organization={Beijing University of Posts and Telecommunications},
			city={Beijing},
			country={China}
		}

\begin{abstract}
Agentic artificial intelligence is increasingly used to decompose industrial tasks, propose robot actions, and adapt execution plans in dynamic cyber-physical environments. However, autonomous proposal generation alone does not guarantee that multi-robot industrial systems preserve task dependencies, resource ownership, safety holds, or repair boundaries during long-horizon execution. This paper introduces a verification-gated agentic mission-state governance framework for intelligent industrial multi-robot systems. The framework maintains two synchronized state objects: an evolving task forest for persistent hierarchy, delayed grounding, and repairable substructures; and a governed blackboard for online execution state, robot traces, resource locks, world beliefs, proposals, verification records, and scene-temporary constraints. From each forest--blackboard snapshot, a derived execution coupling topology exposes cross-branch dependencies for proposal verification, parallel-commit eligibility, and bounded repair. Candidate assignments, repairs, deferrals, and constraint updates may be generated by heuristic, optimization, or agentic reasoning modules, but they can update the committed mission state only after deterministic verification and atomic commit. We evaluate the framework in an indoor factory multi-robot scenario, 30-seed remote-construction stress benchmarks, structural ablations, and scalability probes. The results show improved verified and safety-audited mission-state progress with fewer invalid commitments, lock conflicts, duplicate assignments, abandoned nodes, and disruptive repairs under modeled mission predicates. The study positions agentic AI as a proposal-generating layer governed by inspectable mission-state verification rather than as an unchecked execution authority.
\end{abstract}


\end{frontmatter}

\section{Introduction}

Intelligent industrial systems increasingly combine autonomous robots, industrial cyber-physical infrastructure, digital-twin state models, and artificial-intelligence modules that can perceive, reason, plan, and act in dynamic environments. Smart factories, automated logistics, remote construction, infrastructure inspection, and field maintenance all require multiple robots to coordinate over long horizons while task readiness, robot availability, resource ownership, and operational constraints evolve during execution \cite{gerkey2004formal_mrta,korsah2013comprehensive_mrta,calvo2025heterogeneous_mrta}. In this setting, the central challenge is not only to generate a task allocation or a plan. The system must govern which agentic proposals are allowed to become committed mission state.

This governance problem has become more important as language-model, optimization, and heuristic modules are increasingly used as proposal sources for embodied agents. Such modules can decompose tasks, suggest assignments, explain failures, or propose repairs \cite{ahn2022saycan,huang2023voxposer,rana2023sayplan,vemprala2024chatgpt_robotics,kannan2023smartllm}. Yet an industrial robot team cannot safely treat every generated proposal as executable. A plausible assignment may violate a workstation mutex, residual-energy bound, precedence relation, spatial access condition, temporary protective hold, or already committed robot trace. Agentic AI for industrial systems therefore needs a mission-state layer that separates proposal generation from verified commitment.

Existing multi-robot allocation, symbolic planning, behavior-tree, and feedback-driven execution methods address important parts of this problem, but they rarely maintain persistent task hierarchy, online execution commitments, temporary operational constraints, cross-branch execution couplings, and repair boundaries as synchronized state records. Static task graphs can starve the executable frontier when delayed observations expose new work. Flat blackboard memories can record events without governing which updates remain consistent. Global replanning can recover feasibility, but may rewrite protected work outside the affected region. These limitations motivate a representation in which persistent hierarchy, transient industrial constraints, and verified execution commitments are governed together.

This paper proposes a verification-gated agentic mission-state governance framework for intelligent industrial multi-robot systems. The framework maintains an evolving task forest and a governed blackboard as two canonical state objects. It derives an execution coupling topology from their current snapshot to expose precedence, flow, resource, spatial, synchronization, and protective couplings without storing a third persistent state. Agentic, heuristic, or optimization modules may submit typed proposals, but assignments, deferrals, repairs, diagnostics, and constraint updates are committed only after deterministic verification and atomic state update. Disturbances are localized by hard-edge closure over the affected execution neighborhood.

The main contributions are:
\begin{itemize}
    \item \textbf{Agentic mission-state governance:} a runtime layer that allows autonomous proposal generation while protecting industrial multi-robot mission state through deterministic verification and atomic commit.
    \item \textbf{Task-forest blackboard representation:} a two-object state model that separates persistent task hierarchy from online robot traces, resource locks, world beliefs, verification records, and scene-temporary constraints.
    \item \textbf{Coupling-aware bounded repair:} a derived execution coupling topology that restricts repair to the hard-coupled affected neighborhood while preserving protected work outside that scope.
    \item \textbf{Industrial validation and stress testing:} evaluation in an indoor factory scenario, 30-seed remote-construction benchmarks, structural ablations, and scale probes under modeled mission predicates.
\end{itemize}

\begin{figure*}[!t]
\centering
\includegraphics[width=0.95\textwidth]{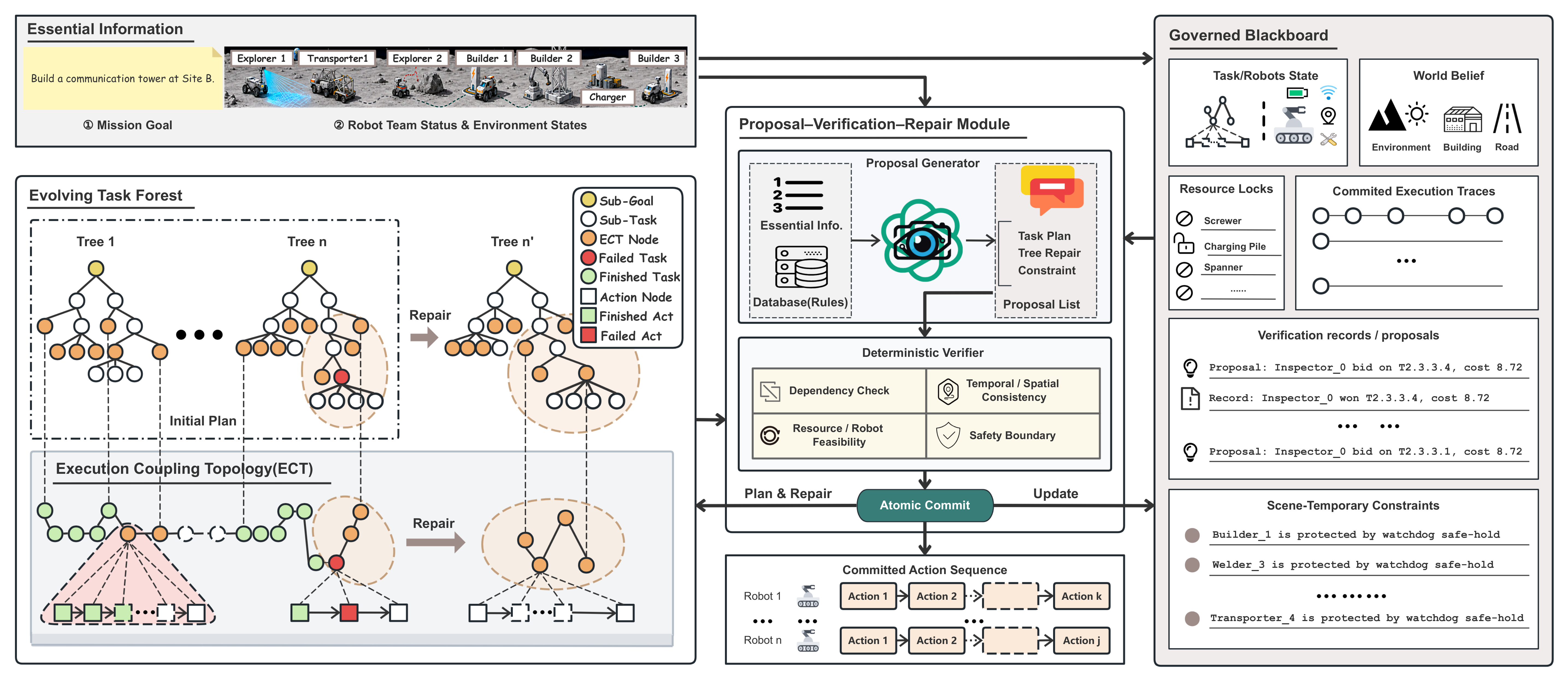}
\caption{Verification-gated agentic mission-state governance framework. Agentic, heuristic, or optimization modules may propose assignments, repairs, deferrals, diagnostics, and constraint updates, but only verified proposals can atomically update the evolving task forest and governed blackboard. The execution coupling topology is derived from the current state snapshot and used for eligibility, verification, and bounded repair.}
\label{fig:framework_structure}
\end{figure*}

\section{Related Work}
\label{sec:related_work}

\subsection{Task-Structured Long-Horizon Multi-Robot Execution}

Long-horizon multi-robot autonomy requires task allocation, scheduling, path feasibility, and execution monitoring to be considered jointly. Classical and recent MRTA methods provide mature assignment tools under capability, time, energy, and coordination constraints \cite{gerkey2004formal_mrta,korsah2013comprehensive_mrta}. Recent studies further incorporate long-duration factors such as limited endurance, recharge, task relaying, coalition execution, dense-workspace coupling, and task-and-motion feasibility \cite{calvo2025heterogeneous_mrta,lai2025roboballet,yang2023piginet}. These methods show that mission performance depends on more than one-shot robot-task matching.

Symbolic and hierarchical planning provide a complementary reference point for long-sequence reasoning. PDDL/PDDL2.1, heuristic-search planners, HTN planners, and robotic planning infrastructures such as ROSPlan encode preconditions, effects, temporal constraints, resources, and decomposition knowledge in explicit models \cite{fox2003pddl21,hoffmann2001ff,helmert2006fast_downward,nau2003shop2,cashmore2015rosplan}. They are highly relevant to long-horizon execution, but they typically assume that the maintained domain model, current task instance, execution state, and update record are already consistent. Our work targets this runtime-maintenance problem: how to keep an evolving task hierarchy, temporary operational constraints, resource locks, committed traces, verification records, and repair boundaries synchronized during execution.

\subsection{LLM-Assisted Multi-Robot Planning and Verifiable Proposals}

LLMs and multimodal foundation models have recently been used for task interpretation, decomposition, grounding, and high-level robot coordination \cite{ahn2022saycan,vemprala2024chatgpt_robotics,driess2023palme,zitkovich2023rt2}. Broader embodied-AI surveys organize this evolution around perception, decision, and execution modules, reinforcing the need to distinguish proposal generation from executable behavior in deployed robots \cite{chen2025embodied_ai_survey}. Situated systems constrain language reasoning through available APIs, 3D value maps, scene graphs, path planners, precondition checks, or iterative feedback \cite{singh2023progprompt,huang2023voxposer,rana2023sayplan}. Multi-robot systems such as SMART-LLM, DART-LLM, COHERENT, LiP-LLM, and EmboTeam are especially close to our setting because they motivate task decomposition, dependency extraction, coalition formation, proposal--feedback execution, and planner or behavior-tree integration \cite{kannan2023smartllm,wang2024dartllm,liu2025coherent,obata2024lipllm,zeng2026emboteam}.

These systems indicate the value of foundation models for semantic decomposition and high-level coordination, but they also show why generated plans should not be committed directly. PlanBench demonstrates that strong LLMs can still fail at systematic planning and reasoning about change \cite{valmeekam2023planbench}; practical robot systems therefore add affordance grounding, precondition checking, simulator feedback, geometric feasibility filtering, or planner interfaces \cite{ahn2022saycan,singh2023progprompt,rana2023sayplan,yang2023piginet}. In multi-robot execution, a semantically plausible proposal may violate capability, residual energy, temporal order, resource exclusivity, safety boundaries, or already committed mission constraints. Recent surveys likewise identify hallucination, latency, scalability, and real-world adaptability as deployment risks for LLM-based multi-robot systems \cite{li2025llm_mrs_survey}. We therefore treat LLM or VLM reasoning as a proposal source rather than an authority: candidate repairs, diagnostic branches, and temporary constraints must pass deterministic mission-state verification before they become committed state.

\subsection{Feedback, Contingency Handling, and Local Repair}

Feedback-driven planning and online repair are necessary because delays, partial failures, blocked traversal, energy shortage, and resource conflicts can invalidate an initially feasible branch. SayPlan and COHERENT use feedback to refine plans or coordinate execution \cite{rana2023sayplan,liu2025coherent}, long-endurance MRTA methods incorporate online repair under dynamic conditions \cite{calvo2025heterogeneous_mrta}, and EmboTeam uses a shared state mechanism to coordinate embodied multi-robot collaboration \cite{zeng2026emboteam}. Classical blackboard architectures provide a broader precedent for coordinating heterogeneous reasoning sources through shared problem-solving state and control \cite{hayesroth1985blackboard_control,nii1986blackboard_model}.

The remaining gap is not simply whether feedback exists, but how feedback changes are governed. Existing feedback loops or blackboard-like memories often serve as communication logs, state caches, or plan-adjustment buffers. They do not necessarily couple persistent task hierarchy, online commitments, temporary operational constraints, derived cross-branch dependencies, and verified repair boundaries. Without this separation, a repair module may update local assignments without updating the task structure, or a graph update may ignore active resource locks, completed segments, protective holds, or committed traces. Our framework makes this coupling explicit by separating persistent hierarchy, governed blackboard records, derived ECT snapshots, and verifier-gated commits.

\subsection{Runtime Governance and Reconfiguration of Large Autonomous Systems}

The proposed framework also connects to systems-engineering work on runtime governance, supervisory control, runtime assurance, and model-based system representation. Supervisory control for discrete-event systems frames control as restricting behavior to satisfy modeled constraints \cite{ramadge1987supervisory_control}. Model-based systems engineering emphasizes explicit system models, interfaces, requirements, and cross-view consistency \cite{madni2018mbse}. Runtime-assurance architectures separate advanced decision components from an assurance layer that can restrict unsafe decisions when modeled constraints are at risk \cite{cofer2020runtime_assurance}.

These ideas motivate explicit models and verification gates, but they do not by themselves specify how a long-horizon multi-robot mission should maintain an evolving task hierarchy, temporary constraints, resource locks, committed traces, derived execution coupling, and disturbance-localized repair in one execution loop. We adapt this systems perspective to mission-state governance. The forest maintains persistent mission structure, the governed blackboard maintains online records and temporary constraints, and the derived ECT exposes the current coupling surface used for deterministic verification and bounded reconfiguration.

\section{Proposed Framework}
\label{sec:method}

\subsection{Problem Setup and Mission-State Governance}
\label{sec:method_overview}

We consider a heterogeneous robot team
\begin{equation}
\mathcal{M}=\{r_1,\ldots,r_N\}
\end{equation}
executing a long-horizon mission initialized from a high-level task description $\mathcal{T}^{0}$ and a partial world state $\mathcal{X}_0$. During execution, robots acquire observations, complete or fail assigned work, occupy and release resources, and encounter scene-dependent operational constraints. The goal is not to compute a static optimal plan at deployment time; it is to maintain a high-level mission state that remains executable, verifiable, and locally repairable as these conditions evolve.

At decision epoch $t$, the governance layer observes
\begin{equation}
\mathcal{S}_t = (\mathcal{F}_t,\mathcal{B}_t,\mathcal{G}_t,\Pi_t,\mathcal{U}_t),
\label{eq:system_state}
\end{equation}
where $\mathcal{F}_t$ is the evolving task forest, $\mathcal{B}_t$ is the governed blackboard, $\mathcal{G}_t$ is the derived ECT, $\Pi_t$ is the candidate proposal set, and $\mathcal{U}_t$ is the exposed execution-unit set. Only $\mathcal{F}_t$ and $\mathcal{B}_t$ are canonical maintained state objects; $\mathcal{G}_t$ is recomputed from their current snapshot.

The framework maintains a modeled consistency invariant
\begin{equation}
I(\mathcal{F}_t,\mathcal{B}_t)=1,
\label{eq:state_invariant}
\end{equation}
covering dependency consistency, robot-capability feasibility, residual energy, temporal consistency, resource exclusivity, temporary constraints, and repair-boundary preservation. These predicates are high-level mission-state checks. The framework assumes that lower-level perception, motion planning, localization, and control modules provide conservative feasibility and execution-feedback records to the blackboard; it does not replace continuous collision avoidance, controller-level safety, or physical verification of task success.

Fig.~\ref{fig:framework_structure} summarizes the framework as four layers: mission inputs and feedback, canonical mission state, derived coupling view, and governed transition. The remainder of this section follows the three technical pillars of the method: the two canonical state objects, the derived ECT, and the verification-gated transition and repair mechanism. Detailed state fields, verifier predicates, and the benchmark assignment cost are provided in Appendix~\ref{app:formal_details}.

\subsection{Canonical Mission State: Task Forest and Governed Blackboard}
\label{sec:canonical_state}

The canonical state representation separates persistent mission structure from online execution records. The evolving task forest $\mathcal{F}_t$ stores hierarchical decomposition, node-level contracts, delayed grounding, and repairable substructures. A node may be open, ready, committed, running, blocked, completed, failed, frozen, or abandoned, but these states are governed transitions rather than free planner outputs. In unfamiliar environments, delayed grounding allows task attributes or descendants to remain unresolved until observations make them reliable. The executable frontier is therefore restricted to grounded nodes whose predecessors and resources are feasible:
\begin{equation}
\begin{aligned}
\mathcal{V}^{\mathrm{front}}_t
=
\{\, v\in\mathcal{V}^{\mathrm{exec}}_t
\mid\;&g_v(t)=\mathrm{grounded},\\
&\mathrm{Pred}(v,t),\\
&\mathrm{Res}(v,t),\\
&\chi_v(t)=\mathrm{ready}\,\}.
\end{aligned}
\label{eq:executable_frontier}
\end{equation}

The governed blackboard $\mathcal{B}_t$ stores the online records needed to evaluate this frontier and to commit verified updates: task-state projections, robot states, world belief, resource locks, committed traces, candidate proposals, verification records, and scene-temporary constraints. The blackboard is governed in three senses. It is node-anchored because execution-relevant state is bound to task-forest nodes; write-restricted because heuristic, optimization, rule-based, or semantic modules may write candidates but not committed state; and atomically committed because a verified update changes the forest and blackboard as one synchronized transition.

Scene-temporary constraints remain on the blackboard rather than being embedded permanently into the forest. Candidate constraints may originate from world-belief updates, resource locks, execution feedback, operator input, or semantic/VLM proposals, but only verified-hard constraints can induce hard ECT couplings. Verified-soft constraints affect ranking or costs, protective constraints block risky commitments until resolved, and safety-critical holds are never bypassed by diagnostic trials. A verified proposal induces the canonical state transition
\begin{equation}
(\mathcal{F}_{t+1},\mathcal{B}_{t+1})
=
\operatorname{Commit}(\mathcal{F}_{t},\mathcal{B}_{t},p),
\label{eq:commit_rule}
\end{equation}
after which the ECT is regenerated rather than stored.

\subsection{Derived Execution Coupling Topology}
\label{sec:forest_derived_ect}

The ECT is a decision-time projection that exposes execution couplings among semantic subtree units. It is derived from the canonical state:
\begin{equation}
\mathcal{G}^{\mathrm{ECT}}_t
=
\Pi_{\mathrm{ECT}}(\mathcal{F}_t,\mathcal{B}_t),
\label{eq:ect_projection}
\end{equation}
and represented as
\begin{equation}
\mathcal{G}^{\mathrm{ECT}}_t
=
(\mathcal{U}_t,\mathcal{R}_t).
\label{eq:ect_graph}
\end{equation}
The units $\mathcal{U}_t$ are selected from an exposed antichain of task-forest subtrees, so a branch and its descendant are not simultaneously treated as independent global execution units. This lets the verifier reason at the semantic level needed for the current decision rather than over every primitive action.

The typed edge set $\mathcal{R}_t$ contains precedence, enabling-flow, resource-mutex, synchronization, spatial-coupling, and soft-preference edges:
\begin{equation}
\mathcal{R}_t =
\mathcal{R}^{\mathrm{pre}}_t\cup
\mathcal{R}^{\mathrm{flow}}_t\cup
\mathcal{R}^{\mathrm{mutex}}_t\cup
\mathcal{R}^{\mathrm{sync}}_t\cup
\mathcal{R}^{\mathrm{spatial}}_t\cup
\mathcal{R}^{\mathrm{soft}}_t .
\label{eq:ect_edges}
\end{equation}
Hard edges are induced only by persistent node contracts and verified-hard blackboard records. Candidate constraints do not enter the ECT, soft edges do not enter hard repair closure, and protective holds are handled as verifier guards. The absence of a hard ECT edge does not by itself authorize concurrent execution; it only makes two units eligible for parallel commitment, after which blackboard-level verification still checks capabilities, energy, resource locks, traces, spatial access, temporary constraints, and safety boundaries.

\subsection{Verification-Gated Commit and Bounded Repair}
\label{sec:verified_repair}
\label{sec:bounded_subforest_repair}

The framework treats assignments, deferrals, reorderings, diagnostic actions, repairs, and constraint updates as typed proposals. Proposal generation is intentionally modular: it may use heuristics, optimization, rule-based repair, or language/VLM reasoning. These modules are proposal sources, not execution authorities. A proposal can modify the committed mission state only if it passes deterministic verification:
\begin{equation}
\begin{aligned}
\Omega(p,\mathcal{F}_t,\mathcal{B}_t,\mathcal{G}_t)=
&\ \Omega_{\mathrm{dep}}
\wedge\Omega_{\mathrm{cap}}
\wedge\Omega_{\mathrm{eng}}\\
&\wedge\Omega_{\mathrm{temp}}
\wedge\Omega_{\mathrm{res}}
\wedge\Omega_{\mathrm{bd}}.
\end{aligned}
\label{eq:proposal_verification}
\end{equation}
The predicate checks forest and ECT dependencies, robot capabilities, residual energy, temporal anchors, resource locks, active verified-hard constraints, safety holds, and repair-boundary consistency. The valid proposal set and selected update are
\begin{equation}
\mathcal{P}^{\mathrm{valid}}_t=
\{p\in\Pi_t\mid \Omega(p,\mathcal{F}_t,\mathcal{B}_t,\mathcal{G}_t)=1\},
\label{eq:valid_proposals}
\end{equation}
\begin{equation}
p^{*}=\arg\min_{p\in\mathcal{P}^{\mathrm{valid}}_t}J(p).
\label{eq:selected_proposal}
\end{equation}
If no valid proposal exists during normal execution, the framework makes no new high-level commitment in that decision cycle.

\begin{assumption}[Sound projected state]
\label{ass:sound_projection}
At the beginning of a high-level decision cycle, the blackboard projection is conservative with respect to the executed system state: completed and running task-node states are current, robot capability labels are correct, residual-energy entries are lower bounds, resource locks include active reservations, active temporary constraints are not optimistic over their guarded scope, and the world-belief quantities used by the verifier are not optimistic over the next commitment horizon.
\end{assumption}

\begin{proposition}[Verified commitment invariance]
\label{prop:verified_commitment}
Under Assumption~\ref{ass:sound_projection}, if a proposal $p$ satisfies $\Omega(p,\mathcal{F}_t,\mathcal{B}_t,\mathcal{G}_t)=1$ and the commit operation in~\eqref{eq:commit_rule} is atomic, then the committed update preserves the hard constraints explicitly checked by the verifier: dependency consistency, capability coverage, residual-energy feasibility, temporal consistency, resource exclusivity, active verified-hard temporary constraints, safety-critical holds, and boundary consistency with completed or protected mission segments.
\end{proposition}
The proof is provided in Appendix~\ref{app:proofs}. The proposition is intentionally limited to represented mission-state predicates; it does not claim physical safety, global optimality, or eventual mission completion.

When a disturbance occurs, the repair routine first localizes the affected region instead of globally replanning. Given disturbed forest nodes $\mathcal{V}^{\mathrm{dist}}_t$, the disturbed units, hard-coupled affected units, and affected subforest are
\begin{equation}
\begin{aligned}
\mathcal{U}^{\mathrm{dist}}_t
&=\{u_i\in\mathcal{U}_t\mid
\operatorname{scope}(u_i)\cap\mathcal{V}^{\mathrm{dist}}_t\neq\emptyset\},\\
\mathcal{U}^{\mathrm{aff}}_t
&=\operatorname{cl}_{\mathcal{R}^{\mathrm{hard}}_t}
(\mathcal{U}^{\mathrm{dist}}_t),\\
\mathcal{F}^{\mathrm{aff}}_t
&=\operatorname{SubtreeClosure}_{\mathcal{F}_t}
\{\alpha(u_i)\mid u_i\in\mathcal{U}^{\mathrm{aff}}_t\}.
\end{aligned}
\label{eq:affected_subforest}
\end{equation}
Repairs are restricted to $\mathcal{F}^{\mathrm{aff}}_t$ and its projected blackboard context. Completed work, protected commitments, resource locks, temporal anchors, downstream interfaces, safety-critical holds, and hard-coupled protected boundary units form the boundary context. If no verified repair or restricted fallback is valid, the affected branch is marked blocked or deferred rather than committing an unchecked update.

\begin{proposition}[ECT-aware bounded repair locality]
\label{prop:bounded_locality}
Assume that $\mathcal{G}^{\mathrm{ECT}}_t$ is derived from a versioned snapshot of $(\mathcal{F}_t,\mathcal{B}_t)$, that $\mathcal{U}^{\mathrm{aff}}_t$ is computed by hard-edge closure as in~\eqref{eq:affected_subforest}, and that a repair proposal modifies only $\mathcal{F}^{\mathrm{aff}}_t$ and its projected blackboard context. If the repair satisfies $\Omega_{\mathrm{bd}}$ and all admitted temporary-constraint updates pass verification, then completed nodes, protected active commitments, resource locks, temporal anchors, safety-critical holds, hard-coupled protected boundary units, and unaffected robot execution traces outside the affected scope remain invariant under the repair commit.
\end{proposition}
The proof is provided in Appendix~\ref{app:proofs}. This result characterizes locality under the current hard ECT closure and boundary predicate, not a guarantee that every disturbance can be repaired.

\begin{algorithm}[!t]
	\footnotesize
	\caption{Governed Task-Forest Blackboard Execution}
	\label{alg:governed_execution}
	\begin{algorithmic}[1]
		\REQUIRE Task forest $\mathcal{F}_t$, blackboard $\mathcal{B}_t$, robot team $\mathcal{M}$
		\ENSURE Updated task forest and blackboard
		
		\WHILE{mission not terminated}
		\STATE Update task, robot, world-belief, resource-lock, and execution-trace records on $\mathcal{B}_t$
		\STATE Apply verified grounding, expansion, completion, or state transitions to $\mathcal{F}_t$
		\STATE Evaluate, expire, or admit temporary constraints on the governed blackboard
		\STATE Derive the exposed execution units and refresh $\mathcal{G}^{\mathrm{ECT}}_t$
		
		\IF{execution disturbance is detected}
		\STATE Localize affected units by hard ECT closure and attempt bounded repair
		\STATE \textbf{continue}
		\ENDIF
		
		\STATE Extract executable frontier $\mathcal{V}^{\mathrm{front}}_t$ using~\eqref{eq:executable_frontier}
		\STATE Generate candidate proposals $\Pi_t$ over $\mathcal{F}_t$, $\mathcal{B}_t$, and $\mathcal{G}^{\mathrm{ECT}}_t$
		\STATE Verify proposals using~\eqref{eq:proposal_verification}
		
		\IF{$\mathcal{P}^{\mathrm{valid}}_t\neq\emptyset$}
		\STATE Select $p^*$ using~\eqref{eq:selected_proposal} and commit it using~\eqref{eq:commit_rule}
		\ELSE
		\STATE Advance execution without a new high-level commitment
		\ENDIF
		\ENDWHILE
	\end{algorithmic}
\end{algorithm}

\section{Experimental Evaluation}
\label{sec:experiments}

\subsection{Experimental Setup}
\label{subsec:experiment_setup}

We evaluate whether verification-gated mission-state governance improves audited progress in intelligent industrial multi-robot systems. The primary industrial validation is an indoor factory scenario with 12 heterogeneous robots, shared workstations, aisle-like resource locks, inspection and rework loops, and cross-tree production dependencies. To test stress behavior beyond the factory topology, we also use remote-construction benchmarks with sparse prior information, delayed grounding, high traversal cost, temporary spatial constraints, and expensive repair. Table~\ref{tab:tro_scenarios} summarizes the generated scenarios.

\begin{table}[!t]
	\centering
	\caption{Benchmark-generation protocol.}
	\label{tab:tro_scenarios}
	\scriptsize
	\setlength{\tabcolsep}{2pt}
	\resizebox{\columnwidth}{!}{%
		\begin{tabular}{lcccl}
			\toprule
			Scenario & Robots & Exec. nodes & Unknown & Disturbance / resource pressure \\
			\midrule
			Small & 6 & 76 & 25\% & Low; no scheduled fault \\
			Medium & 12 & 207 & 45\% & Scheduled faults; moderate obstacles \\
			Large & 20 & 328 & 60\% & Frequent faults; constrained resources \\
			XL probe & 24 & 704 & 70\% & Stress probe; single observed seed \\
			Factory & 12 & 72 & 28--42\% & Workstation mutex; stress sweep \\
			\bottomrule
		\end{tabular}
	}
	\vspace{0.6mm}
	
	\begin{minipage}{0.96\columnwidth}
		\scriptsize\emph{Note:} All scenarios use capability-typed robots, residual-energy checks, terrain/path costs, nonshareable resource locks, delayed grounding, and scheduled execution feedback; rows vary task scale and disturbance pressure.
	\end{minipage}
\end{table}

\begin{figure*}[!t]
\centering
\includegraphics[width=0.95\textwidth]{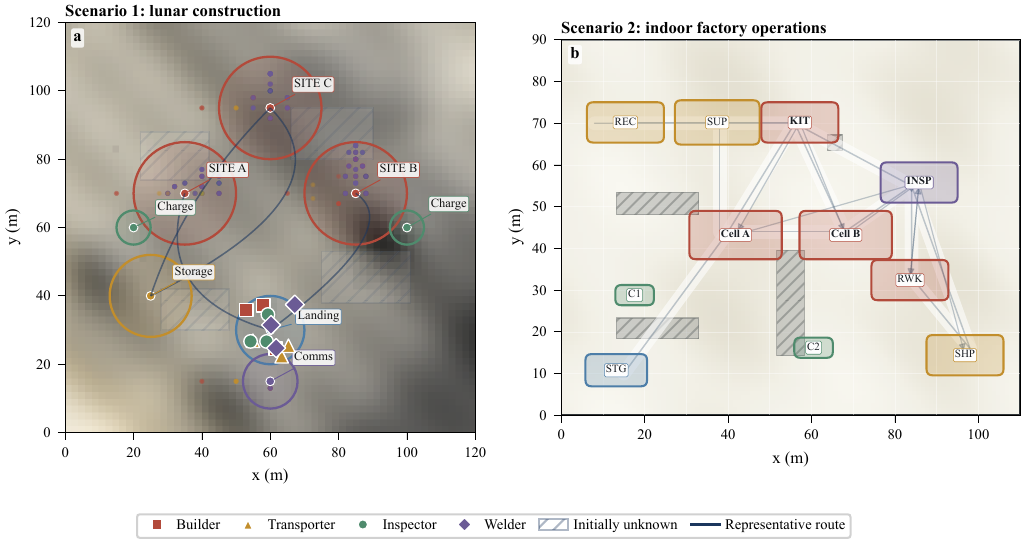}
\caption{Validation domains. Panel (a) shows the remote-construction stress benchmark with terrain cost, resource zones, unknown obstacle regions, and construction sites. Panel (b) shows the indoor factory scenario with replenishment, kitting, assembly, inspection, rework, shipping, and mutex-protected workstations.}
\label{fig:tro_environment}
\end{figure*}

All methods receive the same mission specification, robot capabilities, energy model, initial task forest, and execution feedback stream. We compare classical allocators, simulator reproductions of DART-DAG and LiP-LP, and two internal variants: flat-BB removes the governed ledger, while global-repair replans over all non-frozen tasks instead of using bounded repair. Non-governed methods submit allocator outputs to a shared baseline loop with coarse dependency, capability, energy, passability, duplicate, and lock checks. This makes completion comparable while auditing invalid commitments, abandoned affected work, lock conflicts, and disruption events. Table~\ref{tab:tsmc_baseline_protocol} summarizes the protocol.

\begin{table}[!t]
	\centering
	\caption{Baseline capability protocol.}
	\label{tab:tsmc_baseline_protocol}
	\scriptsize
	\setlength{\tabcolsep}{2pt}
	\resizebox{\columnwidth}{!}{%
		\begin{tabular}{lccccc}
			\toprule
			Method class & Dep. & Audit & Temp. ledger & Verify gate & ECT repair \\
			\midrule
			Ours & Y & Y & Y & Y & Y \\
			DART-DAG & Y & Y & N & N & N \\
			LiP-LP & Y & Y & N & N & N \\
			Greedy/Random & N & Y & N & N & N \\
			Flat-BB variant & Partial & Y & Flat & Y & Y \\
			Global-repair variant & Y & Y & Y & Y & Global \\
			\bottomrule
		\end{tabular}
	}
	\vspace{0.6mm}
	
	\begin{minipage}{0.96\columnwidth}
		\scriptsize\emph{Note:} All methods receive the same initial forest, robot model, execution feedback stream, and shared feasibility audit; only the full framework maintains the governed ledger and ECT-aware repair boundary.
	\end{minipage}
\end{table}

\subsection{Primary Industrial Factory Validation}
\label{subsec:factory_validation}

Table~\ref{tab:tsmc_factory_validation} and Fig.~\ref{fig:tsmc_factory_validation} report the primary industrial validation. The factory scenario contains 72 executable nodes organized into replenishment, kitting, assembly, inspection/rework, and shipping. It tests whether the agentic governance layer can maintain safe progress when multiple robots share workstations, lanes, and cross-tree production dependencies.

\begin{table*}[!t]
	\centering
	\scriptsize
	\caption{Cross-domain indoor factory validation.}
	\label{tab:tsmc_factory_validation}
	\setlength{\tabcolsep}{3pt}
	\resizebox{0.8\textwidth}{!}{%
		\begin{tabular}{lcccccccc}
			\toprule
			Method & Seeds & Safe compl. & Invalid & Lock viol. & Dup. & Viol./100 & Disr. ev. & Time (s) \\
			\midrule
			Ours & 5 & 1.000 $\pm$ 0.000 & 0.0 $\pm$ 0.0 & 0.0 $\pm$ 0.0 & 0.0 $\pm$ 0.0 & 0.0 $\pm$ 0.0 & 0.0 $\pm$ 0.0 & 11.76 $\pm$ 1.43 \\
			Static graph & 5 & 0.178 $\pm$ 0.070 & \textemdash & \textemdash & \textemdash & \textemdash & \textemdash & 1.72 $\pm$ 1.07 \\
			Flat-BB var. & 5 & 0.931 $\pm$ 0.038 & 5.0 $\pm$ 2.7 & 5.4 $\pm$ 3.4 & 5.0 $\pm$ 2.7 & 21.4 $\pm$ 12.1 & \textemdash & 1.99 $\pm$ 0.30 \\
			Global-rep. var. & 5 & 0.775 $\pm$ 0.052 & 1.4 $\pm$ 1.1 & \textemdash & 1.4 $\pm$ 1.1 & 3.9 $\pm$ 3.2 & 42.2 $\pm$ 7.9 & 11.51 $\pm$ 0.91 \\
			DART-DAG & 5 & 0.550 $\pm$ 0.000 & 200.4 $\pm$ 6.8 & 222.0 $\pm$ 6.4 & \textemdash & 586.7 $\pm$ 18.4 & \textemdash & 0.03 $\pm$ 0.00 \\
			LiP-LP & 5 & 0.550 $\pm$ 0.000 & 132.6 $\pm$ 10.6 & 157.4 $\pm$ 6.9 & \textemdash & 402.8 $\pm$ 24.1 & \textemdash & 0.02 $\pm$ 0.00 \\
			\bottomrule
		\end{tabular}
	}
	\vspace{0.6mm}
	\begin{minipage}{0.76\textwidth}
		\scriptsize\emph{Note:} The factory scenario contains 72 executable task nodes, five macro task trees, 16 cross-tree dependencies, and 12 heterogeneous robots. Values are mean $\pm$ standard deviation over five seeds. Raw completion is omitted because it saturates for all non-static methods; dashes mark non-active all-zero diagnostics. Viol./100 aggregates invalid commitments, resource-lock violations, and duplicate assignments per 100 executable tasks.
	\end{minipage}
\end{table*}

\begin{figure*}[!t]
\centering
\includegraphics[width=0.95\textwidth]{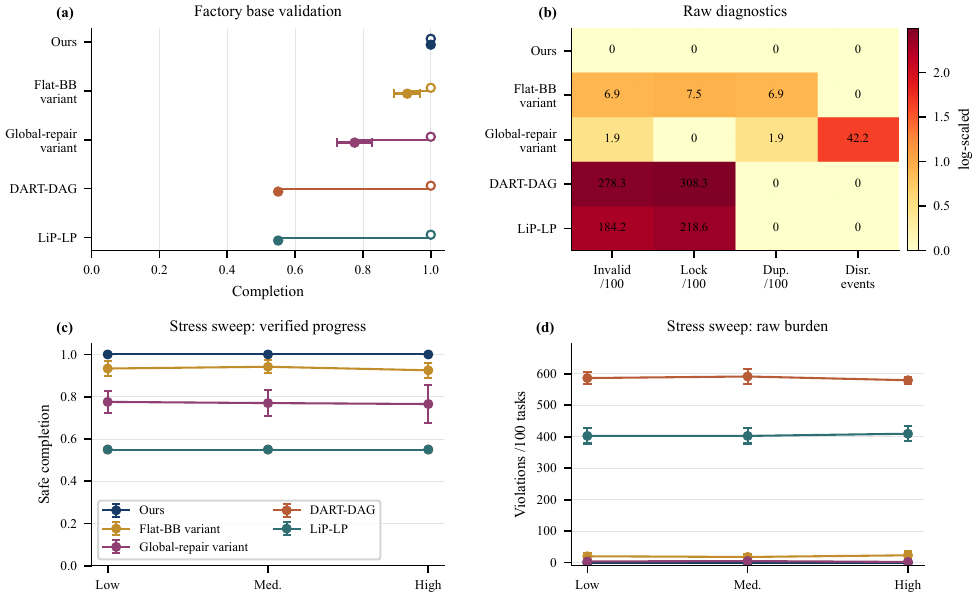}
\caption{Industrial factory validation and stress sweep. Panel (a) compares raw and safety-audited completion in the base factory setting, panel (b) reports raw diagnostics, and panels (c)--(d) show five-seed stress sweeps under increasing disturbance and resource contention.}
\label{fig:tsmc_factory_validation}
\end{figure*}

In the base factory run, the static graph stalls because delayed grounding and structural repair do not refresh the exposed frontier. Several dependency-aware methods reach full raw completion, but their diagnostic records differ: flat-BB incurs duplicate assignments and lock conflicts, DART-DAG and LiP-LP accumulate invalid-commitment and lock burdens, and global-repair reduces some violations by rewriting committed regions, producing disruption node events. The full framework reaches full raw and safety-audited completion with zero recorded invalid, lock, duplicate, and disruption events in the five-seed factory run. The stress sweep preserves the same qualitative pattern as disturbance and resource contention increase.

\subsection{Remote-Construction Stress Comparison}
\label{subsec:construction_stress}

Table~\ref{tab:tro_rq1_main} reports the 30-seed Medium/Large remote-construction stress comparison. The full framework maintains high completion with zero invalid commitments, lock violations, duplicate assignments, abandoned executable nodes, and disruption node events in these modeled runs. The relevant comparison is the joint diagnostic profile rather than raw completion alone: equal displayed completion means indicate equal completed-node counts, not equal mission-state quality.

\begin{table*}[!t]
	\centering
	\scriptsize
	\caption{Main 30-seed audited comparison.}
	\label{tab:tro_rq1_main}
	\setlength{\tabcolsep}{2.3pt}
	\resizebox{\textwidth}{!}{%
		\begin{tabular}{llccccccccc}
			\toprule
			Scenario & Method & Seeds & Raw compl. & Safe compl. & Invalid & Lock & Dup. & Aband. & Disr. ev. & Time (s) \\
			\midrule
			Medium & Ours & 30 & 0.921 $\pm$ 0.172 & 0.921 $\pm$ 0.172 & 0.0 $\pm$ 0.0 & 0.0 $\pm$ 0.0 & 0.0 $\pm$ 0.0 & 0.0 $\pm$ 0.0 & 0.0 $\pm$ 0.0 & 45.47 $\pm$ 7.32 \\
			Medium & DART-DAG & 30 & 0.784 $\pm$ 0.200 & 0.579 $\pm$ 0.338 & 386.8 $\pm$ 787.4 & \textemdash & \textemdash & 32.8 $\pm$ 35.5 & \textemdash & 0.64 $\pm$ 0.26 \\
			Medium & LiP-LP & 30 & 0.767 $\pm$ 0.200 & 0.556 $\pm$ 0.305 & 387.0 $\pm$ 787.6 & \textemdash & \textemdash & 36.5 $\pm$ 34.7 & \textemdash & 0.65 $\pm$ 0.22 \\
			Medium & Greedy & 30 & 0.783 $\pm$ 0.209 & 0.588 $\pm$ 0.325 & 385.5 $\pm$ 784.8 & \textemdash & \textemdash & 33.2 $\pm$ 35.7 & \textemdash & 0.59 $\pm$ 0.24 \\
			Medium & Random & 30 & 0.757 $\pm$ 0.219 & 0.536 $\pm$ 0.370 & 380.0 $\pm$ 774.1 & \textemdash & \textemdash & 40.2 $\pm$ 42.1 & \textemdash & 0.60 $\pm$ 0.23 \\
			Medium & Flat-BB var. & 30 & 0.922 $\pm$ 0.171 & 0.840 $\pm$ 0.157 & 17.0 $\pm$ 5.4 & 1.6 $\pm$ 1.3 & 17.0 $\pm$ 5.4 & \textemdash & \textemdash & 47.91 $\pm$ 8.28 \\
			Medium & Global-rep. var. & 30 & 0.921 $\pm$ 0.172 & 0.624 $\pm$ 0.229 & 19.9 $\pm$ 18.7 & \textemdash & 19.9 $\pm$ 18.7 & \textemdash & 154.7 $\pm$ 146.2 & 41.03 $\pm$ 7.34 \\
			\midrule
			Large & Ours & 30 & 0.898 $\pm$ 0.182 & 0.898 $\pm$ 0.182 & 0.0 $\pm$ 0.0 & 0.0 $\pm$ 0.0 & 0.0 $\pm$ 0.0 & 0.0 $\pm$ 0.0 & 0.0 $\pm$ 0.0 & 316.16 $\pm$ 69.19 \\
			Large & DART-DAG & 30 & 0.595 $\pm$ 0.190 & 0.257 $\pm$ 0.260 & 865.2 $\pm$ 1595.3 & \textemdash & \textemdash & 118.1 $\pm$ 68.1 & \textemdash & 1.29 $\pm$ 0.50 \\
			Large & LiP-LP & 30 & 0.607 $\pm$ 0.189 & 0.272 $\pm$ 0.274 & 865.3 $\pm$ 1595.5 & \textemdash & \textemdash & 114.2 $\pm$ 65.5 & \textemdash & 1.25 $\pm$ 0.40 \\
			Large & Greedy & 30 & 0.589 $\pm$ 0.191 & 0.257 $\pm$ 0.285 & 986.9 $\pm$ 1664.8 & \textemdash & \textemdash & 119.0 $\pm$ 62.7 & \textemdash & 1.23 $\pm$ 0.39 \\
			Large & Random & 30 & 0.555 $\pm$ 0.210 & 0.213 $\pm$ 0.268 & 1096.7 $\pm$ 1705.0 & \textemdash & \textemdash & 123.2 $\pm$ 77.5 & \textemdash & 1.29 $\pm$ 0.44 \\
			Large & Flat-BB var. & 30 & 0.898 $\pm$ 0.182 & 0.816 $\pm$ 0.168 & 27.1 $\pm$ 8.3 & 1.3 $\pm$ 1.2 & 27.1 $\pm$ 8.3 & \textemdash & \textemdash & 364.82 $\pm$ 94.46 \\
			Large & Global-rep. var. & 30 & 0.898 $\pm$ 0.182 & 0.895 $\pm$ 0.181 & 0.3 $\pm$ 1.5 & \textemdash & 0.3 $\pm$ 1.5 & \textemdash & 2.0 $\pm$ 10.8 & 277.37 $\pm$ 57.23 \\
			\bottomrule
		\end{tabular}
	}
	\vspace{0.6mm}
	\begin{minipage}{0.96\textwidth}
		\scriptsize\emph{Note:} Medium and Large scenarios are evaluated after removing the redundant Hungarian relaxation row. Values are mean $\pm$ standard deviation over seeds 42--71. Diagnostic columns are raw event counts; dashes mark all-zero diagnostic channels that are not active failure modes for that row.
	\end{minipage}
\end{table*}

\begin{figure*}[!t]
\centering
\includegraphics[width=0.86\textwidth]{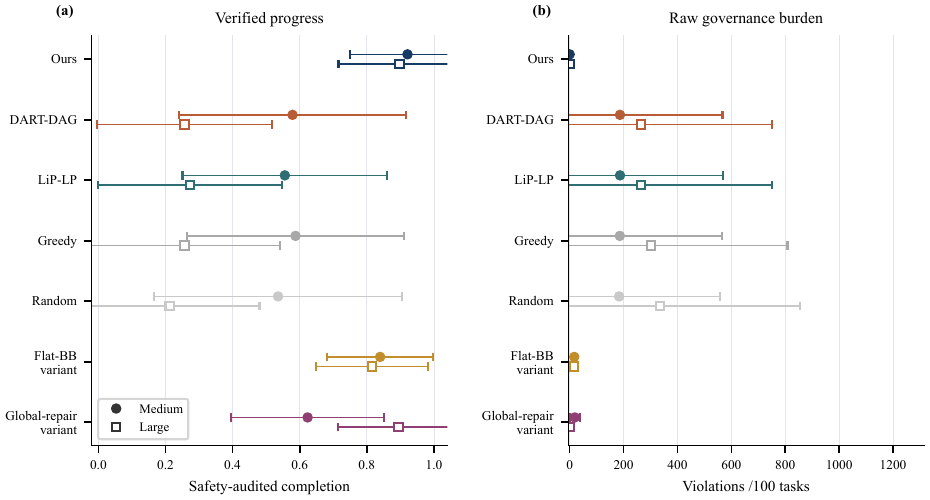}
\caption{Main 30-seed audited remote-construction comparison on Medium and Large. Filled circles denote Medium means and open squares denote Large means; error bars show one standard deviation over seeds 42--71.}
\label{fig:tro_rq1}
\end{figure*}

Fig.~\ref{fig:tro_rq1} shows that the proposed governance layer preserves verified progress with near-zero modeled governance violations. Randomized and greedy allocators are not bare random execution; their outputs pass through the shared feasibility audit. The evidence is therefore not strict dominance on every scalar metric, but the ability to maintain audited mission-state consistency while continuing long-horizon execution.

\begin{figure*}[!t]
\centering
\includegraphics[width=0.86\textwidth]{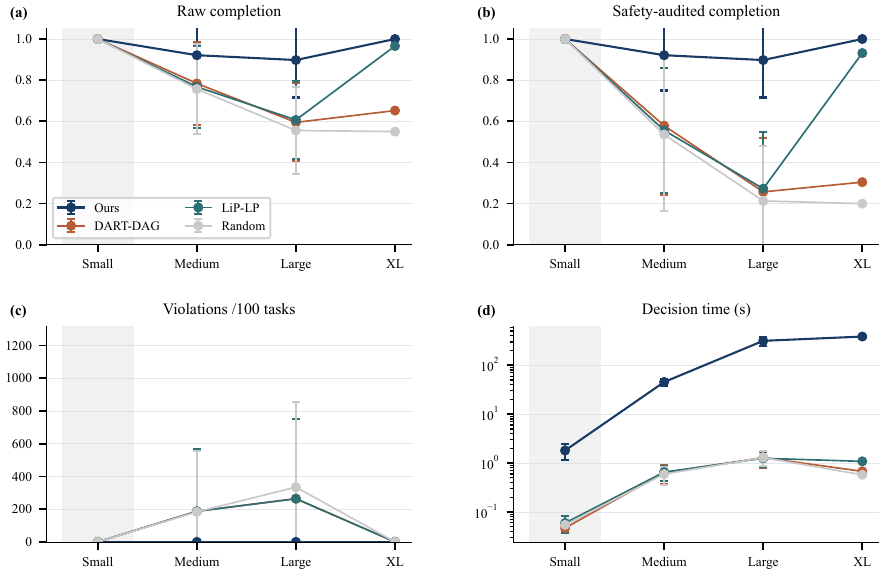}
\caption{Observed task-scale trends across evaluated scenarios. Solid markers connect measured means only; no predicted or extrapolated values are shown. Small uses five seeds, Medium and Large use 30 seeds, and XL is a single observed stress probe.}
\label{fig:tro_task_scale}
\end{figure*}

\subsection{Ablation, Repair, and Runtime Diagnostics}
\label{subsec:ablation_runtime}

Table~\ref{tab:tro_ablation} and Fig.~\ref{fig:tro_ablation} summarize the Medium ablation study. Static graphs stall because hidden executable work is not exposed after obstacle discovery. Removing verification can inflate raw task counts by accepting proposals that should be rejected, so the ablation comparison starts from safety-audited completion and raw diagnostics. Global repair increases disruption events, no structural repair abandons affected downstream nodes, and no verifier admits invalid commitments. The full method is the only variant that combines high audited progress with zero invalid commitments and zero disruption node events.

\begin{table*}[!t]
	\centering
	\caption{Medium-scenario ablation results.}
	\label{tab:tro_ablation}
	\resizebox{0.8\textwidth}{!}{%
		\begin{tabular}{lccccc}
			\toprule
			Method & Safe compl. & Invalid & Abandoned & Disr. ev. & Recovery \\
			\midrule
			Ours & 0.921 $\pm$ 0.172 & 0.0 $\pm$ 0.0 & 0.0 $\pm$ 0.0 & 0.0 $\pm$ 0.0 & 1.3 $\pm$ 1.1 \\
			Static graph & 0.024 $\pm$ 0.008 & \textemdash & \textemdash & \textemdash & \textemdash \\
			Flat-BB var. & 0.845 $\pm$ 0.158 & 16.0 $\pm$ 5.2 & \textemdash & \textemdash & 1.0 $\pm$ 1.0 \\
			No verifier & 0.902 $\pm$ 0.172 & 26.7 $\pm$ 51.4 & \textemdash & \textemdash & 1.2 $\pm$ 1.1 \\
			Global-rep. var. & 0.622 $\pm$ 0.230 & 21.0 $\pm$ 17.3 & \textemdash & 152.3 $\pm$ 135.7 & 1.2 $\pm$ 1.1 \\
			No structural repair & 0.593 $\pm$ 0.351 & \textemdash & 42.5 $\pm$ 51.0 & \textemdash & \textemdash \\
			No risk cost & 0.922 $\pm$ 0.171 & 0.0 $\pm$ 0.0 & 0.0 $\pm$ 0.0 & 0.0 $\pm$ 0.0 & 2.5 $\pm$ 1.4 \\
			\bottomrule
		\end{tabular}
	}
	\vspace{0.6mm}
	\begin{minipage}{0.78\textwidth}
		\footnotesize\emph{Note:} Values are mean $\pm$ standard deviation over 30 seeds in the structural instrumentation suite. Raw completion and energy per completed node are omitted because they are misleading for variants that complete only an easy subset or accept verifier-rejected proposals; dashes mark non-active all-zero diagnostic channels.
	\end{minipage}
\end{table*}

\begin{figure*}[!t]
\centering
\includegraphics[width=0.82\textwidth]{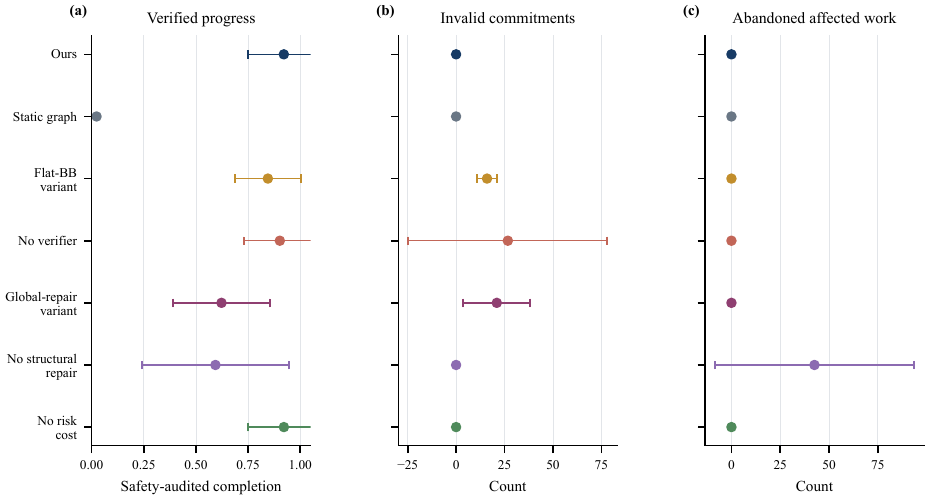}
\caption{Component ablations on the Medium scenario. Panel (a) reports safety-audited completion, panel (b) reports invalid commitments, and panel (c) reports abandoned affected work.}
\label{fig:tro_ablation}
\end{figure*}

\begin{table*}[!t]
	\centering
	\caption{Full-method runtime instrumentation.}
	\label{tab:tro_runtime_breakdown}
	\scriptsize
	\setlength{\tabcolsep}{3pt}
	\resizebox{\textwidth}{!}{%
		\begin{tabular}{lccccccc}
			\toprule
			Scenario & Seeds & Task nodes & Decision epochs & Cand. evals & Cand./epoch & Verify (s) & Total (s) \\
			\midrule
			Small & 30 & 76.1 $\pm$ 0.3 & 76.0 [76.0,77.0] & 156.0 [156.0,156.8] & 1.88 $\pm$ 0.39 & 1.05 $\pm$ 0.39 & 1.37 $\pm$ 0.42 \\
			Medium & 30 & 208.3 $\pm$ 1.1 & 159.0 [151.2,221.8] & 1322.5 [1256.2,1544.0] & 8.18 $\pm$ 1.29 & 16.45 $\pm$ 2.56 & 17.67 $\pm$ 2.67 \\
			Large & 30 & 330.0 $\pm$ 1.2 & 255.0 [248.0,3701.5] & 7960.0 [7880.0,95265.8] & 29.86 $\pm$ 2.82 & 149.78 $\pm$ 35.73 & 153.56 $\pm$ 33.23 \\
			XL & 1 & 707.0 & 391.0 & 13622.0 & 34.84 & 383.61 & 386.76 \\
			\bottomrule
		\end{tabular}
	}
	\vspace{0.6mm}
	\begin{minipage}{0.96\textwidth}
		\scriptsize\emph{Note:} Results are from the repair/scaling suite. Task nodes and time columns report mean $\pm$ standard deviation for multi-seed settings; decision epochs and candidate evaluations report median [IQR]. Event-triggered repair counts are omitted because disturbance frequency is scenario dependent. The XL row is a single observed stress probe.
	\end{minipage}
\end{table*}

\begin{figure*}[!t]
\centering
\includegraphics[width=0.86\textwidth]{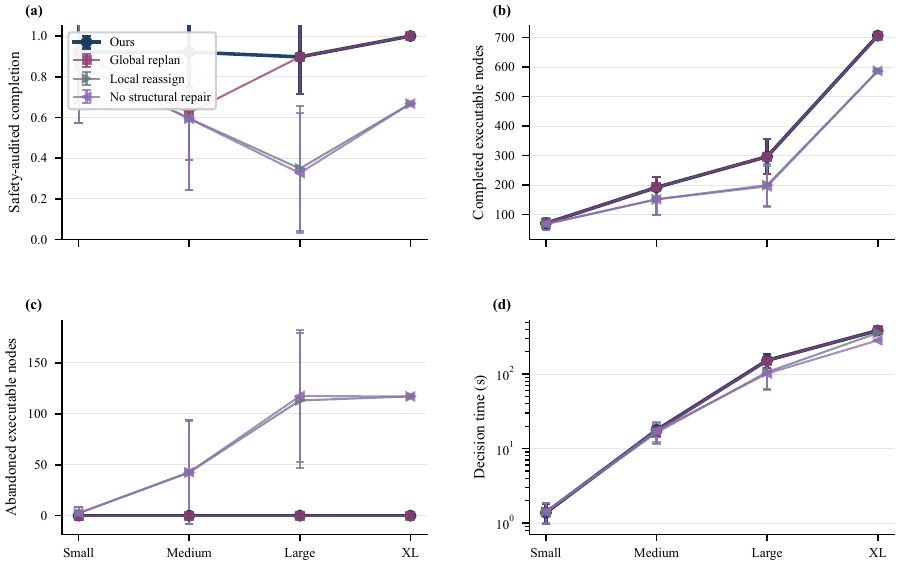}
\caption{Bounded subforest repair versus replanning alternatives across scenario scale. The XL point is a one-seed stress probe; completed and abandoned executable nodes are reported because affected-subforest size is not informative when a method stalls before repair is committed.}
\label{fig:tro_repair}
\end{figure*}

The full method is computationally heavier than simple assignment baselines because it maintains the forest--blackboard projection, verifies proposals, refreshes the execution coupling topology, and localizes repair. The runtime table instruments only the full method and identifies dominant internal costs. Times are Python simulation wall-clock times for mission-state decisions and exclude physical execution, low-level motion planning, sensing, and robot control. Verification dominates because the executable frontier increases the number of candidate proposals, while bounded repair remains local when disturbance closures are local.

\section{Discussion}
\label{sec:discussion}

The results support a bounded claim: agentic proposal generation for industrial multi-robot systems benefits from a deterministic mission-state governance layer. The framework does not claim that every scalar metric is universally improved. Instead, it shows that generated assignments, repairs, and constraint updates can be converted into verified state transitions that preserve modeled hard constraints and local repair boundaries.

The analytical guarantees apply only to represented mission-state predicates under a conservative projected state. They do not guarantee optimal allocation, physical collision avoidance, task success under unmodeled low-level failures, or eventual completion under arbitrary disturbances. In deployment, the blackboard would need conservative inputs from localization, sensing, motion planning, controller-safety, and industrial digital-twin modules. Stale or optimistic lower-layer information would invalidate the assumptions behind the verifier.

The experiments expose four recurrent failure mechanisms. Static graphs suffer frontier starvation when delayed or recovery work cannot be grounded. Flat blackboard state and unverified proposals produce unsafe progress through duplicate assignments, stale locks, missing temporary-constraint updates, or invalid commitments. Local reassignment without structural repair abandons downstream nodes when a blocked branch changes preconditions. Global repair can restore feasibility but disturbs protected work outside the affected subforest. These failures correspond directly to removing the evolving hierarchy, governed blackboard, verifier, or bounded repair invariant.

The current implementation remains a high-level mission-state simulator. It does not yet model actuator faults, localization drift, perception false positives, degraded wireless links, or continuous collision avoidance. The safety-audited completion score is diagnostic rather than normative; deployments should report raw violation and disruption counts and tune aggregate penalties to mission risk. Future work should connect the verifier to low-level motion and sensing certificates, study decentralized blackboard substrates, and evaluate controlled access to real industrial data and hardware-in-the-loop experiments.

\section{Conclusion}

This paper presented a verification-gated agentic mission-state governance framework for intelligent industrial multi-robot systems. The framework separates persistent task hierarchy from online execution records, derives cross-branch coupling for verification and bounded repair, and admits agentic proposals only through deterministic verification and atomic commit. Factory validation, 30-seed remote-construction stress benchmarks, structural ablations, and scale probes show that this architecture improves verified mission-state progress and reduces invalid or disruptive state transitions under modeled mission predicates. The framework should be viewed as a high-level governance layer for agentic industrial autonomy, with guarantees bounded by represented constraints and lower-layer feasibility information.

\appendix\section{Formal State and Verifier Details}
\label{app:formal_details}

This appendix records the formal details that support the condensed method section. They are kept outside the main text to preserve readability, but they specify the state variables, proposal interface, verifier gates, and bounded-repair fallback used in the simulator.

\subsection{State Objects and Mission-State Invariant}

The modeled invariant in~\eqref{eq:state_invariant} is decomposed as
\begin{equation}
I =
I_{\mathrm{dep}}\land
I_{\mathrm{cap}}\land
I_{\mathrm{eng}}\land
I_{\mathrm{temp}}\land
I_{\mathrm{res}}\land
I_{\mathrm{bd}},
\end{equation}
where the terms respectively check dependency consistency, robot-capability feasibility, residual-energy feasibility, temporal consistency, resource exclusivity, and repair-boundary preservation. These predicates are evaluated over the two canonical state objects below; the ECT is derived from them and is not itself committed state.

The task forest is written as
\begin{equation}
\mathcal{F}_t=(\mathcal{V}_t,\mathcal{E}^{\mathrm{h}}_t,\mathcal{E}^{\mathrm{p}}_t,\mathcal{E}^{\mathrm{c}}_t,\chi_t),
\end{equation}
where $\mathcal{E}^{\mathrm{h}}_t$ stores decomposition edges, $\mathcal{E}^{\mathrm{p}}_t$ stores precedence or enabling dependencies, $\mathcal{E}^{\mathrm{c}}_t$ stores declarative mission annotations, and $\chi_t$ stores node states. A node contract is
\begin{equation}
\zeta(v)=(P_v,O_v,R_v,C_v,T_v,Z_v,S_v),
\end{equation}
where the entries denote preconditions, expected outputs, resource requirements, candidate robots or teams, temporal anchors, spatial zones, and compact verification summaries. For grounded executable nodes, task attributes are represented by $\theta_v=\langle x_v,\kappa_v,W_v,e_v,\rho_v,\phi_v\rangle$.

The governed blackboard is written as
\begin{equation}
\mathcal{B}_t=(\mathcal{B}^{\mathrm{task}}_t,\mathcal{B}^{\mathrm{robot}}_t,\mathcal{B}^{\mathrm{world}}_t,\mathcal{B}^{\mathrm{commit}}_t,\mathcal{B}^{\mathrm{prop}}_t,\mathcal{B}^{\mathrm{ver}}_t,\mathcal{B}^{\mathrm{cstr}}_t).
\end{equation}
A scene-temporary constraint record is
\begin{equation}
c=\langle\tau,\operatorname{scope},g,\phi,s,q,\eta,\lambda\rangle,
\end{equation}
where $\tau$ is the constraint type, $\operatorname{scope}$ is the affected node/resource/robot/zone/unit scope, $g$ is the activation guard, $\phi$ is the deterministic check, $s$ is the source, $q$ is the status, $\eta$ stores evidence, and $\lambda$ stores lifetime or invalidation conditions. Status values include candidate, protective, verified-soft, verified-hard, rejected, and expired.

The blackboard also maintains robot committed traces
\begin{equation}
\Gamma_r(t)=(a_{r,1},a_{r,2},\ldots,a_{r,m_r}),
\end{equation}
and node-bound blackboard records $v\leftrightarrow b_v(t)$ containing readiness, owner, resource reservation, predicted interval, diagnostics, and active constraints. This binding is what lets the verifier test a proposal against both hierarchy-level contracts and online execution commitments.

The antichain execution-unit cut used for the ECT satisfies
\begin{equation}
\forall v_i,v_j\in\mathcal{C}^{\mathrm{cut}}_t,\quad
v_i\notin\operatorname{Desc}_{\mathcal{F}_t}(v_j),\quad
v_j\notin\operatorname{Desc}_{\mathcal{F}_t}(v_i),
\end{equation}
with units $u_i=\operatorname{Subtree}_{\mathcal{F}_t}(v_i)$ and anchors $\alpha(u_i)=v_i$. The hard ECT edge set is
\begin{equation}
\mathcal{R}^{\mathrm{hard}}_t =
\mathcal{R}^{\mathrm{pre}}_t\cup
\mathcal{R}^{\mathrm{flow}}_t\cup
\mathcal{R}^{\mathrm{mutex}}_t\cup
\mathcal{R}^{\mathrm{sync}}_t\cup
\mathcal{R}^{\mathrm{spatial}}_t .
\end{equation}
Only verified-hard blackboard constraints can induce hard ECT edges. Candidate constraints remain hypotheses, verified-soft constraints affect ranking or costs, and protective holds are enforced by the verifier rather than promoted into persistent topology.

\subsection{Typed Proposals and Verification Gates}

The framework generalizes bids, repairs, deferrals, diagnostic insertions, and constraint updates into a common proposal interface:
\begin{equation}
p=\left\langle
\alpha,
\mathcal{S},
C,
\Delta\mathcal{F},
\Delta\mathcal{B},
J
\right\rangle,
\end{equation}
where $\alpha$ is the proposal type, $\mathcal{S}\subseteq\mathcal{V}_t$ is the affected task-forest scope, $C\subseteq\mathcal{M}$ is the selected robot or coalition, $\Delta\mathcal{F}$ and $\Delta\mathcal{B}$ are the proposed forest and blackboard updates, and $J$ is the proposal cost. The type belongs to
\begin{equation}
\begin{aligned}
\alpha\in
\{&
\mathrm{assign},\mathrm{defer},\mathrm{swap},\mathrm{insert},\\
&\mathrm{diagnose},\mathrm{relax},
\mathrm{repair},\mathrm{constraint}\}.
\end{aligned}
\end{equation}
Heuristic, optimization, rule-based, LLM, or VLM modules may generate such proposals, but none of them can directly modify $\mathcal{F}_t$ or $\mathcal{B}_t$.

\begin{table*}[!t]
	\centering
	\caption{Deterministic verification predicates before commitment.}
	\label{tab:tro_verifier_predicates}
	\scriptsize
	\setlength{\tabcolsep}{3pt}
	\begin{tabular}{p{0.12\textwidth}p{0.22\textwidth}p{0.29\textwidth}p{0.29\textwidth}}
		\toprule
		Predicate & Check & Main information & Rejects if \\
		\midrule
		$\Omega_{\mathrm{dep}}$ 
		& Dependency / enabling 
		& Forest precedence; node contracts; ECT pre/flow edges; ordering subgraph 
		& Cycles that cannot be resolved, unsatisfied predecessors, frozen dependencies, or invalid enabling outputs. \\
		
		$\Omega_{\mathrm{cap}}$ 
		& Capability coverage 
		& Required capabilities $\kappa_v$; robot capabilities $\kappa_r$; coalition availability 
		& The assigned robot or coalition lacks required capabilities or is unavailable. \\
		
		$\Omega_{\mathrm{eng}}$ 
		& Residual energy 
		& $E_r(t)$; traversal and execution estimates; reserve bound 
		& Any assigned robot falls below the reserve energy bound. \\
		
		$\Omega_{\mathrm{temp}}$ 
		& Temporal consistency 
		& Predicted intervals; robot execution traces; active commitments; temporal anchors; ECT sync edges 
		& Timing violates predecessors, anchors, commitments, robot execution order, or synchronization constraints. \\
		
		$\Omega_{\mathrm{res}}$ 
		& Resource exclusivity 
		& Resource locks; robot execution traces; requirements $\rho_v$; predicted intervals; ECT mutex/spatial edges 
		& Non-shareable resources overlap, a required resource is locked, or hard spatial/resource couplings conflict. \\
		
		$\Omega_{\mathrm{bd}}$ 
		& Repair boundary and safety 
		& Affected subforest; protected commitments; downstream interfaces; verified-hard constraints; safety-critical holds; ECT hard closure 
		& Repair crosses protected boundaries, bypasses safety-critical constraints, or leaves hard-coupled protected units outside the affected scope. \\
		\bottomrule
	\end{tabular}
\end{table*}

For the construction benchmark, assignment proposals use a risk-aware travel model. For robot $r$ and task $v$,
\begin{equation}
E_{\mathrm{trav}}^{(r\rightarrow v)}=
\eta_r^E\sum_{k\in P^{(r\rightarrow v)}}(\mu_E(x_k)+\beta\sigma_E(x_k))\Delta s,
\end{equation}
\begin{equation}
\Delta t_{\mathrm{trav}}^{C}(v)=
\max_{r\in C}\eta_r^T\sum_{k\in P^{(r\rightarrow v)}}(\mu_T(x_k)+\beta\sigma_T(x_k))\Delta s,
\end{equation}
with coalition execution duration
\begin{equation}
\tau_v^C=
\frac{W_v}{\sum_{r\in C}P_r},
\end{equation}
and assignment cost
\begin{equation}
B(C,v)=
\sum_{r\in C}E_{\mathrm{trav}}^{(r\rightarrow v)}+\gamma\tau_v^C.
\end{equation}
Coalitions must cover the required capabilities and satisfy the residual-energy bound:
\begin{align}
\kappa_v&\subseteq\bigcup_{r\in C}\kappa_r,\\
E_{\mathrm{trav}}^{(r\rightarrow v)}+e_v^{(r)}+E_{\mathrm{reserve}}^{(r)}
&\leq E_r(t),\quad \forall r\in C.
\end{align}

Atomic commitment is allowed only after verification:
\begin{equation}
\begin{aligned}
(\mathcal{F}_{t+1},\mathcal{B}_{t+1})
=
\mathrm{Commit}(p,\mathcal{F}_t,\mathcal{B}_t),\\
\text{only if }\quad
\Omega(p,\mathcal{F}_t,\mathcal{B}_t,\mathcal{G}_t)=1 .
\end{aligned}
\end{equation}
Under the modeled predicates, the intended invariant-preservation relation is
\begin{equation}
I(\mathcal{F}_t,\mathcal{B}_t)=1
\land
\Omega(p,\mathcal{F}_t,\mathcal{B}_t,\mathcal{G}_t)=1
\Rightarrow
I(\mathcal{F}_{t+1},\mathcal{B}_{t+1})=1.
\end{equation}
This implication is limited to represented task contracts, blackboard records, ECT couplings, and verifier predicates.

\subsection{Repair Boundary, Diagnostics, and Fallback}

For a disturbance $\epsilon_t$, the implicated forest nodes $\mathcal{V}^{\mathrm{dist}}_t$ induce disturbed units
\begin{equation}
\mathcal{U}^{\mathrm{dist}}_t=
\{u_i\in\mathcal{U}_t\mid
\operatorname{scope}(u_i)\cap\mathcal{V}^{\mathrm{dist}}_t\neq\emptyset\},
\end{equation}
which are expanded by hard-edge closure:
\begin{equation}
\mathcal{U}^{\mathrm{aff}}_t=
\operatorname{cl}_{\mathcal{R}^{\mathrm{hard}}_t}
(\mathcal{U}^{\mathrm{dist}}_t).
\end{equation}
The affected blackboard context is
\begin{equation}
\mathcal{B}^{\mathrm{aff}}_t=
\operatorname{Project}(\mathcal{B}_t,\mathcal{F}^{\mathrm{aff}}_t),
\end{equation}
including robot traces, resource locks, temporary constraints, world beliefs, active commitments, and verifier records associated with the affected subforest and its hard boundary.

The boundary context preserved during repair is
\begin{equation}
\mathcal{C}^{\mathrm{bd}}_t=
\{\mathcal{V}^{\mathrm{done}},
\mathcal{A}^{\mathrm{fixed}},
\mathcal{L}^{\rho},
\mathcal{T}^{\mathrm{anchor}},
\mathcal{I}^{\mathrm{down}},
\mathcal{H}^{\mathrm{safety}}\},
\end{equation}
where the sets denote completed nodes, protected active commitments, resource locks, temporal anchors, downstream interfaces, and safety-critical holds. Temporal anchors require
\begin{equation}
t_{\mathrm{int}}^{\mathrm{start}}\geq
t_{\mathrm{ext}}+\tau_{\mathrm{ext}}
\end{equation}
for an internal repaired task that depends on a fixed external execution segment. Resource conflicts are detected when mutually exclusive resources overlap in predicted execution intervals:
\begin{align}
\rho_v\cap\rho_u&\neq\emptyset,\\
\max(t_v^{\mathrm{start}},t_u^{\mathrm{start}})
&<
\min(t_v^{\mathrm{start}}+\tau_v^C,
t_u^{\mathrm{start}}+\tau_u^{C'}).
\end{align}

Candidate repair proposals are generated only inside the affected subforest:
\begin{equation}
\mathcal{P}^{\mathrm{repair}}_t=
\operatorname{Repair}(\mathcal{F}^{\mathrm{aff}}_t,
\mathcal{B}^{\mathrm{aff}}_t,
\mathcal{C}^{\mathrm{bd}}_t),
\end{equation}
with valid set
\begin{equation}
\mathcal{P}^{\mathrm{repair,valid}}_t=
\{p\in\mathcal{P}^{\mathrm{repair}}_t\mid
\Omega(p,\mathcal{F}_t,\mathcal{B}_t,\mathcal{G}_t)=1\}.
\end{equation}
If no valid repair exists, deterministic fallback optimization is invoked only within the same affected scope:
\begin{equation}
\mathcal{P}^{\mathrm{fb}}_t=
\operatorname{FallbackOpt}(\mathcal{F}^{\mathrm{aff}}_t,
\mathcal{B}^{\mathrm{aff}}_t,
\mathcal{C}^{\mathrm{bd}}_t),
\end{equation}
\begin{equation}
\mathcal{P}^{\mathrm{fb,valid}}_t=
\{p\in\mathcal{P}^{\mathrm{fb}}_t\mid
\Omega(p,\mathcal{F}_t,\mathcal{B}_t,\mathcal{G}_t)=1\}.
\end{equation}
If this set is empty, the affected branch is marked blocked or deferred instead of committing an unchecked fallback. Non-safety-critical cycles may be handled by expiring stale weak constraints, refining one involved unit, or inserting one bounded diagnostic branch; safety-critical holds are never bypassed.

\subsection{Proof Details for the Main Propositions}
\label{app:proofs}

\begin{proof}[Proof of Proposition~\ref{prop:verified_commitment}]
The verification predicate in~\eqref{eq:proposal_verification} is a conjunction of the hard-constraint predicates represented in the current forest, blackboard, and ECT. If $\Omega(p,\mathcal{F}_t,\mathcal{B}_t,\mathcal{G}_t)=1$, each predicate is true on the conservative projected state specified by Assumption~\ref{ass:sound_projection}. Atomic commitment applies the task-forest and blackboard update as one synchronized transition, so no intermediate state can expose a partially updated hierarchy, resource table, robot trace, or constraint ledger. Therefore the post-commit state preserves the modeled hard constraints explicitly checked by the verifier.
\end{proof}

\begin{proof}[Proof of Proposition~\ref{prop:bounded_locality}]
The affected unit set contains the disturbed units and all units reachable through verified-hard ECT couplings. Lifting this set through the anchoring map $\alpha$ yields the only subforest that the repair proposal is permitted to modify. The boundary predicate $\Omega_{\mathrm{bd}}$ rejects any proposal that changes completed nodes, protected commitments, resource locks, temporal anchors, safety-critical holds, hard-coupled protected boundary units, or unaffected traces outside this scope. Because the commit is atomic, no intermediate state can modify external forest structure without the corresponding blackboard update. Therefore the protected external state remains invariant under the committed repair.
\end{proof}

\bibliographystyle{elsarticle-num}
\bibliography{references}

\end{document}